\documentclass[a4paper]{article}

\usepackage{amsmath}
\usepackage{amsfonts}
\usepackage{amssymb}
\usepackage{amsbsy}
\usepackage{theorem}
\usepackage{algorithm}
\usepackage{algorithmic}
\usepackage{mathrsfs}
\usepackage{paralist}
\usepackage{epsfig}
\usepackage{subfigure}
\usepackage{makeidx}
\usepackage{color}
\usepackage{pstricks}
\usepackage{pst-node}
\usepackage{multido}
\usepackage{varioref}
\usepackage{wrapfig} 

\theoremstyle{plain} \newtheorem{remark}{Remark}[section]
\theoremstyle{plain} \newtheorem{definition}{Definition}[section]
\theoremstyle{plain} 
\theoremstyle{plain} 
\theoremstyle{plain} \newtheorem{proposition}{Proposition}[section]
\theoremstyle{plain} \newtheorem{lemma}{Lemma}[section]
\theoremstyle{plain} \newtheorem{assumption}{Assumption}

\newenvironment{proof}[1][Proof]{\begin{trivlist}
\item[\hskip \labelsep {\bfseries #1}]}{\end{trivlist}}
\newcommand{\qed}{\nobreak \ifvmode \relax \else
      \ifdim\lastskip<1.5em \hskip-\lastskip
      \hskip1.5em plus0em minus0.5em \fi \nobreak
      \vrule height0.5em width0.5em depth0.25em\fi}

\def \CA {{\mathcal{A}}}

\def \CO {{\mathcal{O}}}

\def \CS {{\mathcal{S}}}

\def \D {\Delta}

\def\argmax{\mathop{\rm arg\,max}}

\newcommand \E {\mathop{\mbox{\bf{E}}}\nolimits}
\renewcommand \Pr {\mathop{\mbox{\bf{P}}}\nolimits}
\newcommand \given {\mathrel{|}}

\newcommand \hV {\hat{V}}

\newcommand \inftynorm[1]{\|#1\|_\infty}

\newcommand \meas[1]{\mu\left\{#1\right\}}

\makeatletter
\def\equalsfill{$\m@th\mathord=\mkern-7mu
  \cleaders\hbox{$\!\mathord=\!$}\hfill
  \mkern-7mu\mathord=$}
\makeatother
\newcommand \defn {\triangleq}

\newcommand{\policy}{{\pi}}

\newcommand{\plp} {{\policy'}}
\newcommand{\discount}{{\gamma}}

\newcommand{\States}{{\cal S}}
\newcommand{\Actions}{{\cal A}}
\newcommand{\A}{{\cal A}}

\newcommand \policies {\Pi}
\newcommand \Policies {\Pi}

\newcommand \hD {{\hat{\D}}}
\newcommand \hQ {{\hat{Q}}}
\newcommand \tQ {{\tilde{Q}}}

\newcommand{\bX} {\bar{X}}
\newcommand{\hX} {\hat{X}}
\newcommand{\hXn} {\hX_n}
\newcommand{\bY} {\bar{Y}}
\newcommand{\hY} {\hat{Y}}
\newcommand{\hYn} {\hY_n}
\newcommand{\bD} {\bar{\D}}
\newcommand{\hDn} {\hD_n}

\newcommand \aspls {a_{s,\policy}^*}

\newcommand \haspls {\hat{a}_{s,\policy}^*}

\newcommand \Dpl {\D^{\policy}}

\newcommand \Vpl {V^{\policy}}
\newcommand \Vplp {V^{\policy'}}
\newcommand \Vpls {V^{\policy^*}}
\newcommand \Vpli {V^{\policy_i}}
\newcommand \Qpl {Q^{\policy}}

\newcommand \QplT {Q^{\policy,T}}
\newcommand \Qpli {Q^{\policy_i}}
\newcommand \hDpl {\hD^{\policy}}
\newcommand \hVpl {\hV^{\policy}}
\newcommand \hVplp {\hV^{\policy'}}
\newcommand \hQpl {\hQ^{\policy}}
\newcommand \hQplT {\hQ^{\policy,T}}

\newcommand \hQplTc {\hQ^{\policy,T}_c}
\newcommand \tQpl {\tQ^{\policy}}

\newcommand \tQplTi {\tQ^{\policy,T}_{(i)}}

\newcommand \hQplTK {\hQ^{\policy,T}_K}
\newcommand \ndim {d}
\newcommand{\Se}{\States_\epsilon}

\newcommand \oracle {{\sc{Oracle}}}
\newcommand \fixed {{\sc{Fixed}}}
\newcommand \cnt {{\sc{Count}}}

\newcommand \eq {{=}}

\newcommand \comment[1]{} 
\newcommand {\mycite} {\cite}

\newcommand {\yrcite} {\cite}
\usepackage{url}
\urldef{\mailsa}\path|dimitrak@science.uva.nl,lagoudakis@intelligence.tuc.gr|

\begin{document}

\title{\large Algorithms and Bounds\\
for Rollout Sampling Approximate Policy Iteration%
\thanks{This project was partially supported by the ICIS-IAS proejct and the Marie Curie International Reintegration Grant MCIRG-CT-2006-044980 awarded to Michail G.\ Lagoudakis.}
}

\author{Christos Dimitrakakis
\and Michail G.\ Lagoudakis}

\maketitle

\maketitle

\begin{abstract}
  Several approximate policy iteration schemes without value
  functions, which focus on policy representation using classifiers
  and address policy learning as a supervised learning problem, have
  been proposed recently.  Finding good policies with such methods
  requires not only an appropriate classifier, but also reliable
  examples of best actions, covering the state space sufficiently.
  Up to this time, little work has been done on appropriate covering
  schemes and on methods for reducing the sample
  complexity of such methods, especially in continuous state spaces.
  This paper focuses on the simplest possible covering scheme
  (a discretized grid over the state space) and
  performs a sample-complexity comparison between the
  simplest (and previously commonly used) rollout sampling allocation strategy, which allocates
  samples equally at each state under consideration, and an almost as
  simple method, which allocates samples only as needed and
  requires significantly fewer samples.
\end{abstract}

\section{Introduction}
\label{sec:introduction}

Supervised and reinforcement learning are two well-known learning
paradigms, which have been researched mostly independently. Recent
studies have investigated using mature supervised learning methods for
reinforcement
learning~\mycite{lagoudakisICML03,fern2004api,langfordICML05,fern2006api}. Initial
results have shown that policies can be approximately represented
using multi-class classifiers and therefore it is possible to
incorporate classification algorithms within the inner loops of
several reinforcement learning
algorithms~\mycite{lagoudakisICML03,fern2004api,fern2006api}. This viewpoint
allows the quantification of the performance of reinforcement learning
algorithms in terms of the performance of classification
algorithms~\mycite{langfordICML05}.  While a variety of promising
combinations become possible through this synergy, heretofore there
have been limited practical results and widely-applicable algorithms.

Herein we consider approximate policy iteration algorithms, such as those proposed by Lagoudakis and Parr~\mycite{lagoudakisICML03}
as well as Fern et al.~\mycite{fern2004api,fern2006api}, which do not explicitly represent a value
function. At each iteration, a new policy/classifier is produced using
training data obtained through extensive simulation (rollouts) of the
previous policy on a generative model of the process. These rollouts
aim at identifying better action choices over a subset of states in
order to form a set of data for training the classifier representing
the improved policy. The major limitation of these algorithms,
as also indicated by Lagoudakis and Parr~\mycite{lagoudakisICML03}, is the
large amount of rollout sampling employed at each sampled state.  It is
hinted, however, that great improvement could be achieved with
sophisticated management of sampling. We have verified this intuition in a companion
paper~\mycite{dimitrakakis+lagoudakis:mlj2008} that experimentally
compared the original approach of uninformed uniform sampling with various intelligent sampling
techniques.  That paper employed heuristic variants of well-known
algorithms for bandit problems, such as Upper Confidence
Bounds~\mycite{auerMLJ02} and Successive
Elimination~\mycite{evendarJMLR06}, for the purpose of managing rollouts
(choosing which state to sample from is similar to choosing which lever to pull on a bandit machine).
It should be noted, however, that despite the similarity, rollout management has
substantial differences to standard bandit problems and thus general bandits
results are not directly applicable to our case.

The current paper aims to offer a
first theoretical insight into the rollout sampling problem.  This is done
through the analysis of the two simplest sample allocation methods
described in~\mycite{dimitrakakis+lagoudakis:mlj2008}.  Firstly, the
old method that simply allocates an equal, fixed number of samples at each
state and secondly the slightly more sophisticated method of
progressively sampling all states where we are not yet reasonably
certain of which the policy-improving action would be.

The remainder of the paper is organised as follows.
Section~\ref{sec:preliminaries} provides the necessary background,
Section~\ref{sec:algorithms} introduces the proposed algorithms, and
Section~\ref{sec:related_work} discusses related work.
Section~\ref{sec:theory}, which contains an analysis of the proposed
algorithms, is the main technical contribution.

\section{Preliminaries}
\label{sec:preliminaries}

A {\em Markov Decision Process} (MDP) is a 6-tuple $(\States,\A, P, R,
\discount, D)$, where $\States$ is the state space of the process,
$\A$ is a finite set of actions, $P$ is a Markovian transition model
($P(s,a,s')$ denotes the probability of a transition to state $s'$
when taking action $a$ in state $s$), $R$ is a reward function
($R(s,a)$ is the expected reward for taking action $a$ in state $s$),
$\discount\in (0,1]$ is the discount factor for future rewards, and
$D$ is the initial state distribution. A {\em deterministic policy}
$\policy$ for an MDP is a mapping $\policy: \States \mapsto \A$ from
states to actions; $\policy(s)$ denotes the action choice at state
$s$. The value $\Vpl(s)$ of a state $s$ under a policy $\policy$ is
the expected, total, discounted reward when the process begins in
state $s$ and all decisions at all steps are made according to
$\policy$:
\begin{equation}
  \Vpl(s) = E \left[ \sum_{t=0}^{\infty} \discount^t R\big(s_t,\policy(s_t)\big) | s_0 = s, s_t \sim P  \right] \;\;.
\end{equation}
The goal of the decision maker is to find an optimal policy
$\policy^*$ that maximises the expected, total, discounted reward from
all states; in other words, $\Vpls(s) \geq \Vpl(s)$ for all policies $\pi$
and all states $s \in \States$.

{\em Policy iteration} (PI) is an efficient method for deriving an optimal
policy. It generates a sequence $\policy_1$, $\policy_2$, ...,
$\policy_k$ of gradually improving policies, which terminates when
there is no change in the policy ($\policy_k = \policy_{k-1}$);
$\policy_k$ is an optimal policy. Improvement is achieved by computing
$\Vpli$ analytically (solving the linear Bellman equations) and
the action values
\[
\Qpli(s,a) = R(s,a) + \discount \sum_{s'} P(s,a,s') \Vpli(s')\;\;,
\]
and then determining the improved policy as $\policy_{i+1}(s) = \arg\max_{a} \Qpli(s,a)$.

Policy iteration typically terminates in a small number of steps.
However, it relies on knowledge of the full MDP model, exact
computation and representation of the value function of each policy,
and exact representation of each policy. {\em Approximate policy
  iteration} (API) is a family of methods, which have been suggested
to address the ``curse of dimensionality'', that is, the huge growth
in complexity as the problem grows. In API, value functions and
policies are represented approximately in some compact form, but the
iterative improvement process remains the same.  Apparently, the
guarantees for monotonic improvement, optimality, and convergence are
compromised. API may never converge\comment{CD:reference?}, however in
practice it reaches good policies in only a few iterations.

\subsection{Rollout estimates}
Typically, API employs some representation of the MDP model to compute
the value function and derive the improved policy.  On the other hand, the
Monte-Carlo estimation technique of {\em rollouts} provides a way of
accurately estimating $\Qpl$ at any given state-action pair $(s,a)$
without requiring an explicit MDP model or representation of the value
function.  Instead, a generative model of the process (a simulator)
is used; such a model takes a state-action pair $(s,a)$
and returns a reward $r$ and a next state $s'$ sampled from $R(s,a)$
and $P(s,a,s')$ respectively.


A rollout for the state-action pair $(s,a)$ amounts to simulating a
single trajectory of the process beginning from state $s$,
choosing action $a$ for the first step, and choosing actions according
to the policy $\policy$ thereafter up to a certain horizon $T$. If we denote
the sequence of collected rewards during the $i$-th simulated
trajectory as $r_{t}^{(i)}$, $t=0,1,2,\ldots,T-1$, then the rollout
estimate $\hQplTK(s,a)$ of the true state-action value function
$\Qpl(s,a)$ is the observed total discounted reward, averaged over all $K$ trajectories:
\begin{align*}
  \label{eq:rollout-estimate}
  \hQplTK(s,a) &\defn \frac{1}{K}\sum_{i=1}^{K} \tQplTi(s,a) \;,
  &
  \tQplTi(s,a) &\defn \sum_{t=0}^{T-1}
  \discount^{t} r_{t}^{(i)} \;.
\end{align*}
Similarly, we define $\QplT(s,a) = \E\big( \sum_{t=0}^{T-1}
\discount^{t-1} r_{t} \big| a_0\eq a, s_0 \eq s, a_t \sim \policy, s_t \sim P \big)$ to
be the actual state-action value function up to horizon $T$.  As will
be seen later, with a sufficient amount of rollouts and a long horizon $T$, we
can create an improved policy $\pi'$ from $\pi$ at any state $s$,
without requiring a model of the MDP.

\section{Related work}
\label{sec:related_work}

Rollout estimates have been used in the Rollout Classification Policy Iteration (RCPI)
algorithm~\mycite{lagoudakisICML03}, which has yielded promising results
in several learning domains. However, as stated therein, it is
sensitive to the distribution of training states over the state
space. For this reason it is suggested to draw states from the
discounted future state distribution of the improved policy. This
tricky-to-sample distribution, also used by Fern et al.~\mycite{fern2006api}, yields
better results.  One explanation advanced in those studies is the
reduction of the potential mismatch between the training and testing
distributions of the classifier.

However, in both cases, and irrespectively of the sampling
distribution, the main drawback is the excessive computational cost
due to the need for lengthy and repeated rollouts to reach a good
level of accuracy in the estimation of the value function.
In our preliminary experiments with RCPI, it has
been observed that most of the effort is spent where the action value
differences are either non-existent, or so fine that they require a
prohibitive number of rollouts to identify them.  In this paper, we
propose and analyse sampling methods to remove this performance
bottle-neck.  By restricting the sampling distribution to the case of
a uniform grid, we compare the fixed allocation algorithm
({\fixed})~\mycite{lagoudakisICML03,fern2006api}, whereby a large
fixed amount of rollouts is used for estimating the action values in
each training state, to a simple incremental sampling scheme based on
counting ({\cnt}), where the amount of rollouts in each training state
varies. We then derive complexity bounds, which show a clear
improvement using {\cnt} that depends only on the structure of
differential value functions.

We note that Fern et al.~\mycite{fern2006api} presented a related
analysis.  While they go into considerably more depth with respect to
the classifier, their results are not applicable to our framework.
This is because they assume that there exists some real number $\D^* >
0$ which lower-bounds the amount by which the value of an optimal
action(s) under any policy exceeds the value of the nearest
sub-optimal action in any state $s$.  Furthermore, the algorithm they
analyse uses a fixed number of rollouts at each sampled state.  For a
given minimum $\Delta^*$ value over all states, they derive the
necessary number of rollouts per state to guarantee an improvement
step with high probability, but the algorithm offers no practical way
to guarantee a high probability improvement.  We instead derive error
bounds for the fixed and counting allocation algorithms.
Additionally, we are considering continuous, rather than discrete,
state spaces.  Because of this, technically our analysis is much more
closely related to that of Auer et al.~\mycite{auer2007irs}.

%

\section{Algorithms to reduce sampling cost}
\label{sec:algorithms}
The total sampling cost depends on the balance between the number of
states sampled and the number of samples per state.  In the fixed
allocation scheme~\mycite{lagoudakisICML03,fern2006api}, the same
number of $K|\Actions|$ rollouts is allocated to each state in a
subset $S$ of states and all $K$ rollouts dedicated to a single action
are exhausted before moving on to the next action. Intuitively, if the
desired outcome (superiority of some action) in some state can be
confidently determined early, there is no need to exhaust all
$K|\Actions|$ rollouts available in that state; the training data
could be stored and the state could be removed from the pool without
further examination. Similarly, if we can confidently determine that
all actions are indifferent in some state, we can simply reject it
without wasting any more rollouts; such rejected states could be
replaced by fresh ones which might yield meaningful results. These
ideas lead to the following question: can we examine all states in $S$
collectively in some interleaved manner by selecting each time a
single state to focus on and allocating rollouts only as needed?

Selecting states from the state pool could be viewed as a problem akin
to a multi-armed bandit problem, where each state corresponds to an
arm. Pulling a lever corresponds to sampling the corresponding state
once.  By {\em sampling a state} we mean that we perform a single
rollout for each action in that state as shown in
Algorithm~\ref{alg:sample}. This is the minimum amount of information
we can request from a single state.\footnote{It is possible to also
  manage sampling of the actions, but herein we are only concerned
  with the effort saved by managing state sampling.}  Thus, the
problem is transformed to a {\em variant} of the classic multi-armed
bandit problem.  Several methods have been proposed for various
versions of this problem, which could potentially be used in this
context. In this paper, apart from the fixed allocation scheme
presented above, we also examine a simple counting scheme.

\begin{algorithm}[tb]
   \caption{{\sc SampleState}}
   \label{alg:sample}
\begin{algorithmic}
   \STATE {\bfseries Input:}\! state $s$,\! policy $\policy$,\! horizon $T$,\! discount\! factor\! $\discount$
   \FOR{(each $a \in \Actions$)}
   \STATE $(s', r)$ = {\sc Simulate}$(s,a)$
   \STATE $\tQpl(s,a) = r$
   \STATE $x = s'$
   \FOR{$t=1$ {\bf to} $T-1$}
   \STATE $(x', r)$ = {\sc Simulate}$(x,\policy(x))$
   \STATE $\tQpl(s,a) = \tQpl(s,a) + \discount^t r$
   \STATE $x = x'$
   \ENDFOR
   \ENDFOR
   \STATE{\textbf{return} $\tQpl$}
\end{algorithmic}
\end{algorithm}
%

The algorithms presented here maintain an empirical estimate $\hDpl(s)$ of the
marginal difference of the apparently maximal and the second
best of actions. This can be
represented by the marginal difference in $\Qpl$ values in state $s$,
defined as
\[
\Dpl(s) = \Qpl(s,\aspls) - \max_{a \neq \aspls}\Qpl(s,a) \;,
\]
where $\aspls$ is the action that maximises $\Qpl$ in state $s$:
\[
\aspls = \argmax_{a \in \Actions} \Qpl(s,a) \;.
\]
The case of multiple equivalent maximising actions can be easily
handled by generalising to sets of actions in the manner of Fern et
al.~\mycite{fern2006api}, in particular
\begin{eqnarray*}
A^*_{s,\policy} &\defn& \{a \in \Actions : \Qpl(s,a) \geq \Qpl(s,a'),\; \forall a' \in \Actions\}\\
\Vpl_*(s) &=& \max_{a \in \Actions} \Qpl(s,a)\\
\Dpl(s) &=& \left\{
              \begin{array}{ll}
                \Vpl_*(s) - \max_{a \notin A^*_{s,\policy}}\Qpl(s,a), & \hbox{$A^*_{s,\policy}\subset \Actions$} \\
                0, & \hbox{$A^*_{s,\policy}=\Actions$}
              \end{array}
            \right.
\end{eqnarray*}
\label{par:single_best_action}
However, here we discuss only the single best action case to simplify
the exposition.  The estimate $\hDpl(s)$ is defined using the
empirical value function $\hQpl(s,a)$.


\section{Complexity of sampling-based policy improvement}
\label{sec:theory}
Rollout algorithms can be used for policy improvement under certain
conditions.  Bertsekas~\yrcite{Bertsekas:ADP2MPC:2005} gives several
theorems for policy iteration using rollouts and an approximate value
function that satisfies a consistency property.  Specifically,
Proposition 3.1. therein states that the one-step look-ahead policy
$\pi'$ computed from the approximate value function $\hVpl$, has a
value function which is better than the current approximation $\hVpl$, if $\max_{a
  \in \Actions} \E[r_{t+1} + \discount \hVpl(s_{t+1}) | \pi', s_t = s,
a_t = a] \geq \hVpl(s)$ for all $s \in \CS$.  It is easy to see that an
approximate value function that uses only sampled trajectories from a
fixed policy $\pi$ satisfies this property if we have an adequate
number of samples.  While this assures us that we can perform rollouts
at any state in order to improve upon the given policy, it does not
lend itself directly to policy iteration.  That is, with no way to
compactly represent the resulting rollout policy we would be limited
to performing deeper and deeper tree searches in rollouts.

In this section we shall give conditions that allow policy iteration
through compact representation of rollout policies via a grid and a
finite number of sampled states and sample trajectories with a finite
horizon.  Following this, we will analyse the complexity of the fixed
sampling allocation scheme employed in
\mycite{lagoudakisICML03,fern2006api} and compare it with an oracle
that needs only one sample to determine $\aspls$ for any $s \in \CS$ and a
simple counting scheme.

\subsection{Sufficient conditions}

\begin{assumption}[Bounded finite-dimension state space]
  \label{ass:compact-states}
  The state space $\CS$ is a compact subset of $[0,1]^\ndim$.
\end{assumption}
This assumption can be generalised to other bounded state spaces
easily.  However, it is necessary to have this assumption in order to
be able to place some minimal constraints on the search.
\begin{assumption}[Bounded rewards]
  $R(s,a) \in [0,1]$ for all $a\! \in\! \Actions$, $s\! \in\! \States$.
  \label{ass:bounded-reward}
\end{assumption}
This assumption bounds the reward function and can also be generalised
easily to other bounding intervals.

\begin{assumption}[H\"{o}lder Continuity]
\label{ass:continuity}
For any policy $\policy \in \policies$, there exists $L, \alpha \in [0,1]$,
such that for all states $s, s' \in \States$
\[
  |\Qpl(s,a) - \Qpl(s',a)| \leq \frac{L}{2} \inftynorm{s - s'}^\alpha \;\;.
\]
\end{assumption}
This assumption ensures that the value function $\Qpl$ is fairly
smooth.  
It trivially follows in conjunction with
Assumptions~\ref{ass:compact-states} and \ref{ass:bounded-reward} that
$\Qpl, \Dpl$ are bounded {\em everywhere} in $\States$ if they are
bounded for at least one $s \in \States$.  Furthermore, the following
holds:
\begin{remark}
\label{rem:covering}
Given that, by definition, $\Qpl(s, \aspls) \geq \Dpl(s) + \Qpl(s, a)$
for all $a \neq \aspls$, it follows from
Assumption~\ref{ass:continuity} that
\[
  \Qpl(s', \aspls) \geq \Qpl(s', a) \;\;,
\]
for all $s' \in \States$ such that $\inftynorm{s-s'} \leq
\sqrt[\alpha]{\Dpl(s) / L }$.
\end{remark}
This remark implies that the best action in some state $s$ according
to $\Qpl$ will also be the best action in a neighbourhood of states
around $s$.  This is a reasonable condition as there would be no
chance of obtaining a reasonable estimate of the best action in any
region from a single point, if $\Qpl$ could change arbitrarily fast.
We assert that MDPs with a similar smoothness property on their
transition distribution will also satisfy this
assumption. \comment{CD: We probably need something more, but 'Fitted
  Q-iteration in continuous action-space MDPs' seems to have the
  appropriate assumption}

Finally, we need an assumption that limits the total number of
rollouts that we need to take, as states with a smaller $\Dpl$ will
need more rollouts.
\begin{assumption}[Measure]
  \label{ass:measure}
  If $\meas{S}$ denotes the Lebesgue measure of set $S$, then, for any
  $\policy \in \policies$, there exist $M, \beta > 0$ such that
  $\meas{s \in \States: \Dpl(s) < \epsilon} < M \epsilon^\beta$ for
  all $\epsilon > 0$.
\end{assumption}
This assumption effectively limits the amount of times value-function
changes lead to best-action changes, as well as the ratio of states
where the action values are close.  This assumption, together with the
H\"{o}lder continuity assumption, imposes a certain structure on the
space of value functions.  We are thus guaranteed that the value
function of any policy results in an improved policy which is not
arbitrarily complex.  This in turn, implies that an optimal policy
cannot be arbitrarily complex either.

A final difficulty is determining whether there exists some sufficient
horizon $T_0$ beyond which it is unnecessary to go.  Unfortunately,
even though for any state $s$ for which $\Qpl(s,a') > \Qpl(s,a)$,
there exists $T_0(s)$ such that $\QplT(s,a') > \QplT(s,a)$ for all $T
> T_o(s)$, $T_0$ grows without bound as we approach a point where the
best action changes.  However, by selecting a fixed, sufficiently
large rollout horizon, we can still behave optimally with respect to
the true value function in a compact subset of $\CS$.
\begin{lemma}
  For any policy $\policy \in \Policies$, $\epsilon > 0$, there exists
  a finite $T_\epsilon > 0$ and a compact subset $\Se
  \subset \States$ such that
  \[
  \QplT(s,\aspls) \geq \QplT(s,a)
  \quad
  \forall a \in \Actions,
  s \in \States,  T > T_\epsilon
  \]
  where $\aspls \in \Actions$ is such that $\Qpl(s,\aspls) \geq
  \Qpl(s,a)$ for all $a \in \Actions$.
  \label{prop:finite-horizon}
\end{lemma}
\begin{proof}
  From the above assumptions it follows directly that for any
  $\epsilon > 0 $, there exists a compact set of states $\Se \subset
  \States$ such that $\Qpl(s,\aspls) \geq \Qpl(s,a') + \epsilon$ for
  all $s \in \Se$, with $a' = \argmax_{a \neq \aspls} \Qpl(s,a)$.  Now
  let $x_T \defn \QplT(s,\aspls) - \QplT(s,a')$.  Then, $x_\infty
  \defn \lim_{T \to \infty} x_T \geq \epsilon$.  For any $s \in \Se$
  the limit exists and thus by definition $\exists T_\epsilon(s)$ such
  that $x_{T_\epsilon} > 0$ for all $T > T_\epsilon$.  Since $\Se$ is
  compact, $T_\epsilon \defn \sup_{s \in \Se} T_\epsilon(s)$ also
  exists.\footnote{For a discount factor $\gamma < 1$ we can
    simply bound $T_\epsilon$ with
    $\log[\epsilon(1-\gamma)]/\log(\gamma)$.}  \qed
\end{proof}
This ensures that we can identify the best action within $\epsilon$,
using a finite rollout horizon, in most of $\States$.  Moreover,
$\meas{\Se} \geq 1 - M{2\epsilon}^\beta$ from
Assumption~\ref{ass:measure}.

In standard policy iteration, the improved policy $\pi'$ over $\pi$
has the property that the improved action in any state is the action
with the highest $\Qpl$ value in that state.  However, in
rollout-based policy iteration, we may only guarantee being within
$\epsilon > 0$ of the maximally improved policy.
\begin{definition}[$\epsilon$-improved policy]
  An $\epsilon$-improved policy $\pi'$ derived from $\pi$ satisfies
  \begin{equation}
    \label{eq:epsilon-improvement}
    \max_{a \in \Actions} \Qpl(s,a) - \epsilon \leq \Vplp(s),
  \end{equation}
\end{definition}
Such a policy will be said to be {\em improving in $S$} if $\Vpl(s)
\leq \Vplp(s)$ for all $s \in S$.  The measure of states for which
there can not be improvement is limited by
Assumption~\ref{ass:measure}.  Finding an improved $\pi'$ for the
whole of $\CS$ is in fact not possible in finite time, since this
requires determining the boundaries in $\CS$ at which the best action
changes.  \footnote{To see this, consider $\CS \defn [0,1]$, with some
  $s^* : R(s,a_1) \geq R(s,a_2)$ $\forall s \geq s^*$ and $R(s,a_1) <
  R(s,a_2)$ $\forall s < s^*$.  Finding $s^*$ requires a binary
  search, at best.}

In all cases, we shall attempt to find the improving action $\aspls$
at each state $s$ on a uniform grid of $n$ states, with the next
policy $\plp(s')$ taking the estimated best action $\haspls$ for the
state $s$ closest to $s'$, i.e. it is a nearest-neighbour classifier.

In the remainder, we derive complexity bounds for achieving an
$\epsilon$-improved policy $\pi'$ from $\pi$ with probability at least
$1 - \delta$.  We shall always assume that we are using a sufficiently
deep rollout to cover $\Se$ and only consider the number of rollouts
performed.  First, we shall derive the number of states we need to
sample from in order to guarantee an $\epsilon$-improved policy, under
the assumption that at each state we have an {\em oracle} which can
give us the exact $\Qpl$ values for each state we examine.  Later, we
shall consider sample complexity bounds for the case where we do not
have an oracle, but use empirical estimates $\hQplT$ at each state.

\subsection{The {\oracle} algorithm}

\begin{algorithm}[tb]
   \caption{\oracle}
   \label{alg:oracle_grid}
\begin{algorithmic}
   \STATE{\textbf{Input:} $n$, $\policy$}
   \STATE Set $S$ to a uniform grid of $n$ states in $\CS$.
   \FOR{$s \in S$}
   \STATE $\haspls = \aspls$
   \ENDFOR
   \STATE{\textbf{return} $\hat{A}_{S,\policy}^* \defn \{\haspls : s \in S\}$}
\end{algorithmic}
\end{algorithm}

\comment{CD: Clarify that we select the action that is greedy in the
  middle of the sphere} \comment{CD: Extend to grid-less algorithms?}
\comment{CD: Use simple rather than cumulative regret -- see Pure
  Exploration for Multi-Armed Bandit Problems}

Let $B(s,\rho)$ denote the infinity-norm sphere of radius $\rho$
centred in $s$ and consider Alg.~\ref{alg:oracle_grid} ({\oracle})
that can instantly obtain the state-action value function for any
point in $\CS$.  The algorithm creates a uniform grid of $n$ states,
such that the distance between adjacent states is $2\rho =
\frac{1}{n^{1/\ndim}}$ -- and so can cover $\CS$ with spheres
$B(s,\rho)$.  Due to Assumption~\ref{ass:continuity}, the error in the
action values of any state in sphere $B(s,\rho)$ of state s will be
bounded by $L\left(\frac{1}{2n^{1/\ndim}}\right)^{\alpha}$. Thus, the
resulting policy will be
$L\left(\frac{1}{2n^{1/\ndim}}\right)^{\alpha}$-improved, i.e. this
will be the maximum regret it will suffer over the maximally improved
policy.

To bound this regret by $\epsilon$, it is sufficient to have $n =
\left(\frac{1}{2}\sqrt[\alpha]{\frac{L}{\epsilon}}\right)^\ndim$
states in the grid.  The following proposition follows directly.
\begin{proposition}
  Algorithm~\ref{alg:oracle_grid} results in regret $\epsilon$ for $n
  = \CO\left( L^{\ndim/\alpha} \left[2\epsilon^{1/\alpha}
    \right]^{-\ndim} \right)$.
\end{proposition}
Furthermore, as for all $s$ such that $\Dpl(s) > L\rho^\alpha$,
$\aspls$ will be the improved action in all of $B(s,\rho)$, then
$\pi'$ will be improving in $S$ with $\meas{S} \geq 1 - ML^\beta
\left(\frac{1}{2n^{1/\ndim}}\right)^{\alpha \beta}$.  Both the
regret and the lack of complete coverage are due to the fact that we
cannot estimate the best-action boundaries with arbitrary precision in
finite time.  When using rollout sampling, however, even if we
restrict ourselves to $\epsilon$ improvement, we may still make an
error due to both the limited number of rollouts and the finite
horizon of the trajectories.  In the remainder, we shall derives error
bounds for two practical algorithms that employ a fixed grid with a
finite number of $T$-horizon rollouts.

\subsection{Error bounds for states}
When we estimate the value function at each $s \in S$ using
rollouts there is a probability that the estimated best action $\haspls$
is not in fact the best action.  For any given state under
consideration, we can apply the following well-known lemma to obtain a
bound on this error probability
\begin{lemma}[Hoeffding inequality]
  \label{lem:hoeffding}
  Let $X$ be a random variable in $[b,b+Z]$ with $\bar{X} \defn \E[X]$,
  observed values $X_1, \ldots, X_n$ of $X$, and $\hat{X}_n \defn
  \frac{1}{n}\sum_{i=1}^n X_i$.  Then, $\Pr(\hXn \geq \bX +
  \epsilon) = \Pr(\hXn \leq \bX + \epsilon) \leq \exp
  \left(-2n\epsilon^2/Z^2\right)$ for any $\epsilon > 0$.
\end{lemma}
Without loss of generality, consider two random variables $X, Y \in
[0,1]$, with empirical means $\hXn, \hYn$ and empirical difference
$\hDn \defn \hXn - \hYn > 0$.  Their means and difference will be
denoted as $\bX, \bY, \bD \defn \bX - \bY$ respectively.

Note that if $\bX > \bY$, $\hXn > \bX - \bD/2$ and $\hYn < \bY +
\bD/2$ then necessarily $\hXn > \hYn$, so
$\Pr(\hXn > \hYn | \bX > \bY) \geq \Pr(\hXn > \bX - \bD/2 \wedge \hYn < \bY +
\hDn/2)$.  The converse is
\begin{subequations}
  \label{eq:hoeffding_twovar_one}
  \begin{align}
    \Pr\left(\hXn < \hYn \given \bX > \bY\right)
    &\leq \Pr\left(\hXn < \bX - \bD/2 \vee \hYn > \bY + \bD/2\right)\\
    &\leq \Pr\left(\hXn < \bX - \bD/2\right) + \Pr\left(\hYn > \bY + \bD/2\right)\\
    &\leq 2 \exp\left(-\frac{n}{2} \bD^2\right).
  \end{align}
\end{subequations}

Now, consider $\haspls$ such that $\hQpl(s,\haspls) \geq \hQpl(s,a)$
for all $a$.  Setting $\hXn=Z^{-1} \hQpl(s,\haspls)$ and $\hYn=Z^{-1}
\hQpl(s,a)$, where $Z$ is a normalising constant such that $Q \in
[b,b+1]$, we can apply \eqref{eq:hoeffding_twovar_one}. Note that the
bound is largest for the action $a'$ with value closest to $\aspls$,
for which it holds that $\Qpl(s,\aspls) - \Qpl(s,a') = \Dpl(s)$.
Using this fact and an application of the union bound, we conclude
that for any state $s$, from which we have taken $c(s)$ samples, it
holds that:
\begin{equation}
  \Pr[\exists \haspls \neq \aspls :
  \hQpl(s,\haspls) \geq \hQpl(s,a)] \leq 2|\Actions| \exp\left(- \frac{c(s)}{2Z^2}\Dpl(s)^2\right).
  \label{eq:error_bound}
\end{equation}

\begin{algorithm}[tb]
   \caption{\fixed}
   \label{alg:uniform_grid}
\begin{algorithmic}
   \STATE{\textbf{Input:} $n$, $\policy$, $c$, $T$, $\delta$}
   \STATE Set $S$ to a uniform grid of $n$ states in $\CS$.
   \FOR{$s \in S$}
   \STATE Estimate $\hQplTc(s,a)$ for all $a$.
   \IF{$\hDpl(s) > Z \sqrt{\frac{2\log (2n|\Actions|/\delta)}{c}}$}
   \STATE $\haspls = \argmax{\hQpl}$
   \ELSE
   \STATE $\haspls = \pi(s)$
   \ENDIF
   \ENDFOR
   \STATE{\textbf{return} $\hat{A}_{S,\policy}^* \defn \{\haspls : s \in S\}$}
\end{algorithmic}
\end{algorithm}


\subsection{Uniform sampling: the {\fixed} algorithm}
As we have seen in the previous section, if we employ a grid of $n$
states, covering $\States$ with spheres $B(s,\rho)$, where $\rho =
\frac{1}{2n^{1/\ndim}}$, and taking action $\aspls$ in each
sphere centred in $s$, then the resulting policy $\policy'$ is only
guaranteed to be improved within $\epsilon$ of the optimal improvement
from $\policy$, where $\epsilon = L\rho^\alpha$.  Now, we examine the
case where, instead of obtaining the true $\aspls$, we have an
estimate $\haspls$ arising from $c$ samples from each action in each
state, for a total of $cn|\Actions|$
samples. Algorithm~\ref{alg:uniform_grid} accepts
(i.e. it sets $\haspls$ to be the empirically highest value action in that state) for all states
satisfying:
\begin{equation}
	\hDpl(s) \geq Z \sqrt{\frac{2\log (2n|\CA|/\delta)}{c}}.
	\label{eq:termination_condition}
\end{equation}
The condition ensures that the probability that $\Qpl(s,\haspls) <
\Qpl(s,\aspls)$, meaning the optimally improving action is not
$\haspls$, at any state is at most $\delta$.  This can easily be seen
by substituting the right hand side of
\eqref{eq:termination_condition} for $\epsilon$ in
\eqref{eq:error_bound}.  As $\Dpl(s) > 0$, this results in an error
probability of a single state smaller than $\delta/n$ and we can use a
union bound to obtain an error probability of $\delta$ for each policy
improvement step.

For each state $s \in S$ that the algorithm considers, the following two cases
are of interest:
\begin{inparaenum}[(a)]
\item $\Dpl(s) < \epsilon$, meaning that even when we have correctly
  identified $\aspls$, we are still not improving over all of
  $B(s,\rho)$ and
\item $\Dpl(s) \geq \epsilon$.
\end{inparaenum}


While the probability of accepting the wrong action is always bounded
by $\delta$, we must also calculate the probability that we fail to
accept an action at all, when $\Dpl(s) \geq \epsilon$ to estimate the
expected regret.  Restating our acceptance condition as $\hDpl(s) \geq
\theta$, this is given by:
\begin{align}
	\Pr[\hDpl(s) < \theta] &= \Pr[\hDpl(s) - \Dpl(s) < \theta - \Dpl(s)]
	\nonumber
	\\
	&= \Pr[\Dpl(s) - \hDpl(s) > \Dpl(s) - \theta], \quad \Dpl(s) > \theta.
	\label{eq:insufficient_evidence_probability}
\end{align}

Is $\Dpl(s) > \theta$?  Note that for $\Dpl(s) > \epsilon$, if
$\epsilon > \theta$ then so is $\Dpl$.  So, in order to achieve total
probability $\delta$ for all state-action pairs in this case, after
some calculations, we arrive at this expression for the regret
\begin{equation}
\epsilon =
\max
\left\{
L\left(\frac{1}{2n^{1/\ndim}}\right)^{\alpha},
Z\sqrt{\frac{8\log(2n|\Actions|/\delta)}{c}}
\right\}.
\label{eq:grid_resolution_bound}
\end{equation}
By equating the two sides, we get an expression for the minimum number
of samples necessary per state:
\[
c =
8 \frac{Z^2}{L^2} 4^\alpha n^{2\alpha/\ndim} \log(2n|\Actions|/\delta).
\]
This directly allows us to state the following proposition.
\begin{proposition}
  The sample complexity of Algorithm~\ref{alg:uniform_grid} to achieve
  regret at most $\epsilon$ with probability at least $1-\delta$ is
  $\CO \left( \epsilon^{-2} L^{\ndim/\alpha}\left[2\epsilon^{1/\alpha}
    \right]^{-d} \log \frac{2|\CA|}{\delta} L^{\ndim/\alpha}
    \left[2\epsilon^{1/\alpha}\right]^{-d} \right)$.
\end{proposition}

\subsection{The {\cnt} algorithm}

The {\cnt} algorithm starts with a policy $\pi$ and a set
of states $S_0$, with $n = |S_0|$.  At each iteration $k$, each sample
in $S_k$ is sampled once.  Once a state $s \in S_k$ contains a
dominating action, it is removed from the search.  So,
\[
S_{k} = \left\{s \in S_{k-1} : \hDpl(s) < Z \sqrt{\frac{
    \log(2n|\Actions|/\delta)}{c(s)}}\right\}
\]
 Thus, the number of samples from each state is $c(s) \geq k$
if $s \in S_k$.

We can apply similar arguments to analyse {\cnt}, by noting that the
algorithm spends less time in states with higher $\Dpl$ values.  The
measure assumption then allows us to calculate the number of states
with large $\Dpl$ and thus, the number of samples that are needed.

\begin{algorithm}[tb]
   \caption{\cnt}
   \label{alg:counting_grid}
\begin{algorithmic}
   \STATE{\textbf{Input:} $n$, $\policy$, $C$, $T$, $\delta$}
   \STATE Set $S_0$ to a uniform grid of $n$ states in $\CS$, $c_1, \ldots, c_n = 0$.
   \FOR{$k=1, 2, \ldots$}
       \FOR{$s \in S_k$}
           \STATE Estimate $\hQplTc(s,a)$ for all $a$, increment $c(s)$
           \STATE $S_{k} = \left\{ s \in S_{k-1} :  \hDpl(s) < Z \sqrt{\frac{2 \log(2n|\Actions|/\delta)}{c(s)}} \right\}$
       \ENDFOR
       \IF{$\sum_s c(s) >= C$}
       \STATE Break.
       \ENDIF
   \ENDFOR
\end{algorithmic}
\end{algorithm}

We have already established that there is an upper bound on the regret
depending on the grid resolution $\epsilon < L\rho^\alpha$.
We proceed by forming subsets of states
$W_m = \{ s \in S : \Dpl(s) \in [2^{-m}, 2^{1-m}\}$.  Note that we
only need to consider $m < 1 + \frac{1}{\log 1/2}(\log L + \alpha \log \rho)$.

Similarly to the previous algorithm, and due to our acceptance
condition, for each state $s \in W_m$, we need $c(s) \geq 2^{2m + 1}
Z^2 \log\frac{2n|\Actions|}{\delta}$ in order to bound the total error
probability by $\delta$. The total number of samples necessary is
\[
Z^2 \log\frac{2n|\Actions|}{\delta} \sum_{m=0}^{\lceil\frac{1}{\log
    1/2}(\log L + \alpha \log \rho)\rceil} |W_m| 2^{2m + 1}.
\]
A bound on $|W_m|$ is required to bound this expression.  Note that
\begin{equation}
\meas{B(s,\rho) : \Dpl(s') < \epsilon \forall s' \in B(s,\rho)}
\leq
\meas{s : \Dpl(s) < \epsilon}
< M \epsilon^\beta.
\end{equation}
It follows that $|W_m| < M2^{\beta(1-m)}\rho^{-\ndim}$ and consequently
\begin{align}
  \sum_{s \in S} c(s) &= Z^2 \log\frac{2n|\Actions|}{\delta}
  \sum_{m=0}^{\lceil\frac{1}{\log 1/2}(\log L + \alpha \log
    \rho)\rceil} M2^{\beta(1-m)}\rho^{-\ndim} 2^{2m + 1}\nonumber
\\
&\leq
M 2^{\beta+1}
2^{\frac{1 + \frac{1}{\log 1/2}(\log L + \alpha \log \rho)}{2-\beta}}
2^{\ndim} Z^2 n \log\frac{2n|\Actions|}{\delta}.
\end{align}
The above results directly in the following proposition:
\begin{proposition}
  The sample complexity of Algorithm~\ref{alg:counting_grid} to
  achieve regret at most $\epsilon$ with probability at least
  $1-\delta$, is $\CO \left( L^{\ndim/\alpha}
    \left[2\epsilon^{1/\alpha} \right]^{-d} \log \frac{2|\CA|}{\delta}
     L^{\ndim/\alpha}
    \left[2\epsilon^{1/\alpha}\right]^{-d} \right)$.
\end{proposition}
We note that we are of course not able to remove the dependency on
$d$, which is only due to the use of a grid.  Nevertheless, we obtain
a reduction in sample complexity of order $\epsilon^{-2}$ for this
very simple algorithm.

%

\section{Discussion}
\label{sec:conclusion}

We have derived performance pounds for approximate policy improvement
without a value function in continuous MDPs.  We compared the usual
approach of sampling equally from a set of candidate states to the
slightly more sophisticated method of sampling from all candidate
states in parallel, and removing a candidate state from the set as
soon as it was clear which action is best.  For the second algorithm,
we find an improvement of approximately $\epsilon^{-2}$. Our results
complement those of Fern et al~\mycite{fern2006api} for relational
Markov decision processes. However significant amount of future work
remains.

Firstly, we have assumed everywhere that $T>T_\epsilon$.  While this
may be a relatively mild assumption for $\gamma < 1$, it is
problematic for the undiscounted case, as some states would require
far deeper rollouts than others to achieve regret $\epsilon$.  Thus,
in future work we would like to examine sample complexity in terms of
the depth of rollouts as well.

Secondly, we would like to extend the algorithms to increase the
number of states that we look at: whenever $\hVpl(s) \approx
\hVplp(s)$ for all $s$, then we could increase the resolution.  For
example if,
\[
\sum_{s \in S} \Pr\left(\hVpl(s) + \epsilon< \hVplp(s) \given \Vpl(s) > \Vplp(s)\right) < \delta
\]
then we could increase the resolution around those states with the
smallest $\Dpl$.  This would get around the problem of having to select $n$.

A related point that has not been addressed herein, is the choice of
policy representation.  The grid-based representation probably makes
poor use of the available number of states.  For the
increased-resolution scheme outlined above, a classifier such as
$k$-nearest-neighbour could be employed.  Furthermore, regularised
classifiers might affect a smoothing property on the resulting policy,
and allow the learning of improved policies from a set of states
containing erroneous best action choices.

As far as the state allocation algorithms are concerned, in a
companion paper~\mycite{dimitrakakis+lagoudakis:mlj2008}, we have
compared the performance of {\cnt} and {\fixed} with additional
allocation schemes inspired from the UCB and successive elimination
algorithms.  We have found that all methods outperform {\fixed} in
practice, sometimes by an order of magnitude, with the UCB variants
being the best overall.

For this reason, in future work we plan to perform an analysis of such
algorithms. A further extension to deeper searches, by for example
managing the sampling of actions within a state, could also be
performed using techniques similar
to~\mycite{ECML:Kocsis+Szepesvari:2006}.

\subsection{Acknowledgements}
Thanks to the reviewers and to Adam Atkinson, Brendan Barnwell, Frans
Oliehoek, Ronald Ortner and D. Jacob Wildstrom for comments and useful
discussions.

\bibliography{ewrl2008}
\bibliographystyle{plain}


\end{document}